\documentclass[a4paper, 12pt]{article}

\usepackage[a4paper,top=3cm,bottom=3cm,left=3cm,right=3cm]{geometry}
\usepackage{lmodern}

\usepackage[english]{babel}
\usepackage[T1]{fontenc}
\usepackage[utf8]{inputenc}

\usepackage{authblk}

\usepackage[normalem]{ulem}

\usepackage{csquotes}
\usepackage[hidelinks]{hyperref}	
\usepackage{url}
\hypersetup{    
	colorlinks = true,
	linkcolor  = black,
	urlcolor   = gray,
	citecolor  = black
}
\usepackage{breakcites}

\usepackage[acronym]{glossaries}
\makeglossaries

\usepackage{enumitem}
\usepackage{wasysym}

\usepackage{graphicx}
\usepackage{subcaption}
\usepackage{rotating}
\usepackage{wrapfig}
\counterwithin{figure}{section}

\usepackage[table,xcdraw,usenames,dvipsnames]{xcolor}
\usepackage{booktabs}
\usepackage{adjustbox}
\counterwithin{table}{section}

\usepackage{amssymb}
\usepackage{amsmath}

\usepackage{amsthm}

\theoremstyle{definition}
\newtheorem{definition}{Definition}[section]

\theoremstyle{plain}

\newtheorem{lemma}[definition]{Lemma}
\newtheorem{corollary}[definition]{Corollary}

\theoremstyle{remark}
\newtheorem{remark}[definition]{Remark}

\theoremstyle{plain}

\allowdisplaybreaks
\usepackage{bbold}

\usepackage{algorithm} 
\usepackage{algpseudocode}

\usepackage[super]{nth}
\usepackage{lipsum}

\newglossaryentry{formula}{name=formula,
                           description={A mathematical expression}}

\newacronym{wrt}{w.r.t.}{with respect to}
\newacronym{EU}{EU}{European Union}
\newacronym{iid}{iid}{independent and identically distributed}
\newacronym{iff}{iff}{if and only if}
\newacronym{wlog}{w.l.o.g.}{without loss of generality}
\newacronym{PDE}{PDE}{partial differential equation}
% ...
\newacronym{CDF}{CDF}{cumulative distribution function}
% ...
\newacronym{OP}{OP}{optimization problem}
\newacronym{AI}{AI}{artificial intelligence}
\newacronym{XAI}{XAI}{explainable artificial intelligence}
\newacronym{ML}{ML}{machine learning}
\newacronym{DL}{DL}{deep learning}
\newacronym{GDL}{GDL}{geometric deep learning}
\newacronym{SOTA}{SOTA}{state of the art}
\newacronym{TP}{TP}{true positives}
\newacronym{FP}{FP}{false positives}
\newacronym{FN}{FN}{false negatives}
\newacronym{TN}{TN}{true negatives}
\newacronym{ACC}{ACC}{accuracy}                     
\newacronym{TPR}{TPR}{true positive rate}           
\newacronym{FPR}{FPR}{false positive rate}          
\newacronym{FNR}{FNR}{false negative rate}          
\newacronym{TNR}{TNR}{true negative rate}           
\newacronym{PPV}{PPV}{positive predictive value}    
\newacronym{FDR}{FDR}{false discorvery rate}        
\newacronym{FOR}{FOR}{false omission rate}          
\newacronym{NPV}{NPV}{negative predictive value}    
\newacronym{ROC}{ROC-curve}{receiver operating characteristic curve}
\newacronym{AUC}{AUC}{area under the (ROC) curve}
\newacronym{IG}{IG}{information gain}
\newacronym{MSE}{MSE}{mean squared error}
\newacronym{MAE}{MAE}{mean absolute error}
\newacronym{MRAE}{MRAE}{mean relative absolute error}
\newacronym{DI}{DI}{disparate impact}
\newacronym{DP}{DP}{demographic parity}
\newacronym{EOs}{EOs}{equalized odds}
\newacronym{EO}{EO}{equal opportunity}
\newacronym{SVM}{SVM}{support vector machine}
\newacronym{MLP}{MLP}{multi layer perceptron}
\newacronym{NN}{NN}{neural network}
\newacronym{GNN}{GNN}{graph neural network}
\newacronym{GCN}{GCN}{graph convolutional network}
\newacronym{PI-GCN}{PI-GCN}{physics-informed graph convolutional network}
\newacronym{MPNN}{GCN}{message passing neural network}
% ...
\newacronym{WL}{WL}{Weisfeiler-Lehmann}
\newacronym{WDN}{WDN}{water distribution network}
\newacronym{WDS}{WDS}{water distribution system}
\newacronym{ERC}{ERC}{European Research Council}

\newcommand{\N}{\mathbb{N}}                             
\newcommand{\R}{\mathbb{R}}                             
\newcommand{\1}{\mathbb{1}}                             
\newcommand{\F}{\mathcal{F}}                            
\newcommand{\Prob}{\mathbb{P}}                          
\newcommand{\Exp}{\mathbb{E}}                           

% ...
\newcommand{\Hyp}{\mathcal{H}}                          
\newcommand{\D}{\mathcal{D}}                            
\newcommand{\X}{\mathcal{X}}                            
\newcommand{\Y}{\mathcal{Y}}                            
\newcommand{\Sen}{\mathcal{S}}                          
% ...
% ...

% ...

% ...

\DeclareMathOperator{\Cov}{Cov}
% ...

\DeclareMathOperator{\ACC}{\acrshort{ACC}}              
\DeclareMathOperator{\TPR}{\acrshort{TPR}}              
\DeclareMathOperator{\FPR}{\acrshort{FPR}}              

\DeclareMathOperator{\DI}{\acrshort{DI}}

\DeclareMathOperator{\EO}{\acrshort{EO}}
% ...
\DeclareMathOperator{\sgd}{sgd}                         
% ...

\title{
    {\Large 
    \textbf{Fairness-Enhancing Ensemble Classification in Water Distribution Networks}}
}
\author{
    \small
   Janine Strotherm and Barbara Hammer
}
\affil{
    \small
    Center for Cognitive Interaction Technology, Bielefeld University, Bielefeld, Germany
    \\
    \texttt{\{jstrotherm,bhammer\}@techfak.uni-bielefeld.de}
}
\date{
    \small 
}

\begin{document}

\maketitle


\begin{abstract}
As relevant examples such as the future criminal detection software \cite{Angwin2016MLinCriminalJustice_MachineBias} show, fairness of AI-based and social domain affecting decision support tools constitutes an important area of research. In this contribution, we investigate the applications of AI to socioeconomically relevant infrastructures such as those of water distribution networks (WDNs), where fairness issues have yet to gain a foothold. To establish the notion of fairness in this domain, we propose an appropriate definition of protected groups and group fairness in WDNs as an extension of existing definitions. We demonstrate that typical methods for the detection of leakages in WDNs are unfair in this sense. Further, we thus propose a remedy to increase the fairness which can be applied even to non-differentiable ensemble classification methods as used in this context.
\end{abstract}

 Leakage Detection in Water Distribution Networks



\setcounter{section}{-1} 

\setcounter{page}{1} 
\pagenumbering{arabic}



\section{Introduction}
Due to the increasing usage of AI-based decision making systems in socially relevant fields of application the question of \textit{fair decision making} gained importance in recent years (cf. \cite{Angwin2016MLinCriminalJustice_MachineBias}, \cite{EuropeanComission2012LegalAspects_GuidelinesTrustworthyAI}). 
Fairness is hereby related to the several (protected) groups or individuals which are affected by the algorithmic decision making and characterized by \textit{sensitive features} such as gender or ethnicity. 
Most algorithms on which these tools are based on rely on data which can be biased with respect to the questions of fairness without intention, resulting in skewed models. 
Also the algorithm itself can discriminate protected groups or individuals without explicitly aiming to do so due to an undesirable algorithmic bias (cf. \cite{Mehrabi2021Overview_SurveyOnFairness,Pessach2022Overview_ReviewOnFairness}).

Several definitions of fairness have been discussed (cf. \cite{Barocas2019Overview_FairnessBook,Castelnovo2022Overview_ClarificationOnFairness,Dwork2012FairnessThroughAwareness,Mehrabi2021Overview_SurveyOnFairness,Pessach2022Overview_ReviewOnFairness,Zafar2017FairClassification_InProcess_CovarianceConstraints}). 
From a legal perspective, one distinguishes between \textit{disparate treatment} and \textit{disparate impact} (cf. \cite{Barocas2019Overview_FairnessBook}). While disparate treatment occurs whenever a group or an individual is intentionally treated differently because of their membership in a protected class, disparate impact is a consequence of indirect discrimination happening despite \enquote{seemingly neutral policy} (cf. \cite{Pessach2022Overview_ReviewOnFairness}). 
From a scientific viewpoint, the variety of fairness notions is much larger where many popular approaches focus mainly on (binary) classification tasks (cf. \cite{Castelnovo2022Overview_ClarificationOnFairness,Mehrabi2021Overview_SurveyOnFairness,Pessach2022Overview_ReviewOnFairness}). 

Besides the definition of fairness, the problem arises how to enhance fairness in well known AI methods while maintaining a reasonable overall performance. 
Approaches can hereby be grouped in pre-processing, in-processing and post-processing techniques (cf. \cite{Pessach2022Overview_ReviewOnFairness}), whereas we will focus on in-processing methods.

The question of fairness becomes especially relevant when the decisions of a \gls{ML} model take impact on socioeconomic infrastructure, such as \glspl{WDN}. To the best of our knowledge, the question of fairness has not been approached within this domain. We address the important problem of leakage detection in \glspl{WDN} and investigate in how far typical models treat different groups of consumers of the \gls{WDN} (in)equally. We hereby focus on \textit{group fairness}, which in contrast to \textit{individual fairness} focuses on treating different groups among the \gls{WDN} equally instead of treating similar individuals similarly (cf. \cite{Mehrabi2021Overview_SurveyOnFairness}). 

To come up with a first approach to improve fairness in such a domain of high social and ethical relevance, based on \cite{Zafar2017FairClassification_InProcess_CovarianceConstraints}, we consider the empirical covariance between the sensitive features and the model's prediction as a proxy for the fairness measure. 
The algorithms can handle even multiple non-binary sensitive features and satisfy both the concept of disparate treatment \textit{and} disparate impact simultaneously (cf. \cite{Zafar2017FairClassification_InProcess_CovarianceConstraints}). 
Then, we extend \cite{Zafar2017FairClassification_InProcess_CovarianceConstraints} by 
(a) giving explicit generalized definitions of well-known fairness measures even for multiple non-binary sensitive features, 
(b) modifying their idea to any possibly non-convex classification model including arbitrary ensemble classifier, and 
(c) presenting a method to handle the problem of potential non-differentiability of AI algorithms.

The rest of the work is structured as followed: 
In section \ref{section_FairnessInMachineLearning}, we introduce two definitions of group fairness for multiple non-binary sensitive features. 
Afterwards, in section \ref{section_LeakageDetectionInWaterDistributionNetworks}, we present a standard methodology to detect leakages in \glspl{WDN} and investigate whether the resulting model makes fair decisions with respect to the previously defined notions of fairness. 
Then, in section \ref{section_FairnessEnhancingLeakageDetectionInWaterDistributionNetworks}, we propose and evaluate several adaptations to this methodology that enhance fairness. 
Finally, our findings are summarized and discussed in section \ref{section_Conclusion}.

The implementation of our results 
can be found on GitHub\footnote{\url{https://github.com/jstrotherm/FairnessInWDNS}}.

\section{Fairness in Machine Learning}
\label{section_FairnessInMachineLearning}

Let $(\Omega, \F, \Prob)$ be a probability space of interest, consisting of the sample space $\Omega \neq \emptyset$, a $\sigma$-field $\F$ and a probability measure $\Prob$. Moreover, let 
$\hat{Y}: \Omega \rightarrow \Y$ be some binary classifier, i.e., $\Y = \{0,1\}$, being trained to model some true labels $Y: \Omega \rightarrow \Y$. Usually, $\hat{Y}$ can be written as some model $f: \X \rightarrow \{0,1\}$, applied to the features $X: \Omega \rightarrow \X$, i.e., $\hat{Y} = f(X)$ holds. 
In recent years, the interest towards the question of such classifier $\hat{Y}$ being fair with respect to some additional, sensitive feature $S: \Omega \rightarrow \Sen$ has risen. Mostly, $\Sen = \{0,1\}$ gives binary information about the membership or non-membership of a protected class, such as some certain gender or ethnicity (cf. \cite{Pessach2022Overview_ReviewOnFairness}). 
While the majority of the literature focuses on a single binary sensitive feature $S$ (cf. \cite{Castelnovo2022Overview_ClarificationOnFairness,Mehrabi2021Overview_SurveyOnFairness,Pessach2022Overview_ReviewOnFairness}), 
in this work, we generalize the understanding of fairness to multiple binary sensitive features $S_1, ..., S_K$ 
that model a single non-binary or $K$ different (non-)binary sensitive feature(s).

Within this work, we will focus on group fairness. Assuming that all of the following conditional probabilities exist, one well-known notion of group fairness based on the predictor $\hat{Y}$ and the binary sensitive feature $S$ is called \gls{DI}, requiring that
\begin{align*}
    \frac{
    \Prob(\hat{Y} = 1 ~|~ S = 0)
    }{
    \Prob(\hat{Y} = 1 ~|~ S = 1)
    }
    \geq
    1 - \epsilon
\end{align*}

\noindent is satisfied for some given value $\epsilon \in [0,1]$, assuming that $\{S=0\}$ is the protected group and that the nominator is smaller than the denominator (cf. \cite{Pessach2022Overview_ReviewOnFairness}). 
The disparate impact notion is 
\enquote{designed to mathematically represent the legal notion of \textit{disparate impact}} (cf. \cite{Pessach2022Overview_ReviewOnFairness}). 
It 
assures that the relative amount of positive predictions within the protected group $\{S = 0\}$ deviates at most $100\epsilon\%$ from the relative amount of positive predictions within the non-protected group $\{S = 1\}$.

We generalize this definition to multiple binary sensitive features by
\begin{align}
\label{align_FairnessMeasure_DisparateImpact_Multi}
\begin{split}
    \DI :=
    \min_{ \substack{ k_1, k_2 \in \{1,...,K\} } }
    \frac{
    \Prob(\hat{Y}=1 ~|~ S_{k_1}=1)
    }{
    \Prob(\hat{Y}=1 ~|~ S_{k_2}=1)
    }
    \geq
    1 - \epsilon.
\end{split}
\end{align}

\noindent Criticism of the disparate impact score \gls{DI} could be the missing dependence on the true label $Y$ (cf. \cite{Hardt2016FairClassification_PostProcess_RetrainModel}). 
We thus introduce another notion of group fairness, called \gls{EO}. In standard definition, equal opportunity holds whenever
\begin{align*}
    \left| 
    \Prob(\hat{Y}=1 ~|~ S=0, Y=1) - \Prob(\hat{Y}=1 ~|~ S=1, Y=1) 
    \right| 
    \leq 
    \epsilon
\end{align*}

\noindent is satisfied for some given value $\epsilon \in [0,1]$ (cf. \cite{Mehrabi2021Overview_SurveyOnFairness,Pessach2022Overview_ReviewOnFairness}). 
Equal opportunity ensures the \glspl{TPR} among protected and non-protected groups to differ at most $100\epsilon$\%.

Similarly, we generalize this definition to multiple binary sensitive features:
\begin{align}
\label{align_FairnessMeasure_EqualOpportunity_Multi}
\begin{split}
    \EO :=
    \max_{ \substack{ k_1, k_2 \\ \in \{1,...,K\} } }
    \left| 
    \Prob(\hat{Y}=1 ~|~ S_{k_1}=1, Y=1) 
    - 
    \Prob(\hat{Y}=1 ~|~ S_{k_2}=1, Y=1)
    \right|
    \leq 
    \epsilon.
\end{split}
\end{align}

\begin{remark}
    Our generalized notions of fairness go hand in hand with the conventional ones: In the conventional definitions, a single binary random variable $S$ gives information about the membership of a protected group $\{S = 0\}$ or the membership of a non-protected group $\{S = 1\}$. Our definition handles the existence of $K$ different groups without defining which of the groups are protected in advance. By defining the protected group $\{S = 0\}$ as group 1 and the non-protected group $\{S = 1\}$ as group 2 and the random variables $S_k$ giving information about the membership ($S_k = 1$) or non-membership ($S_k = 0$) of group $k$ for $k=1,2$, the conventional definitions and our definitions (cf. eq. \eqref{align_FairnessMeasure_DisparateImpact_Multi} and \eqref{align_FairnessMeasure_EqualOpportunity_Multi}) coincide, because $\{S=0\} = \{S_1=1\}$ and $\{S=1\} = \{S_2=1\}$ holds.
\end{remark}

\section{Leakage Detection in Water Distribution Networks}
\label{section_LeakageDetectionInWaterDistributionNetworks}

A key challenge in the domain of \glspl{WDN} is to detect leakages. In this task, $\Omega$ corresponds to possible states of a \gls{WDN}, given by time-dependant demands of the end users of the $D$ nodes in the network. We assume that among those, $d$ nodes are provided with sensors (usually, $D \gg d$), which deliver pressure measurements $p(t) \in \R^d$ for different times $t \in \R$ and which can be used for the task at hand. As we usually measure pressure values within fixed time intervals $\delta \in \R_+$, we introduce the notation $t_i := t_0 + i\delta$, where $t_0$ is some fixed reference point with respect to time.

\subsection{Methodology}
\label{subsection_Methodology_In_LeakageDetectionInWaterDistributionNetworks}

There are several methodologies that make use of pressure measurements to approach the problem of leakage detection using \gls{ML}, i.e., by training a classifier $\hat{Y} \in \{0,1\} = \Y$ that predicts the true state of the \gls{WDN} $Y \in \Y$ with respect to the question whether a leak is active (1) or not (0). One standard approach comes in two steps: In first instance, so called \textit{virtual sensors} are trained, i.e., regression models being able to predict the pressure at a given node $j \in \{1,...,d\}$ and some time $t_i \in \R$, based on measured pressure at the remaining nodes $\hat{\jmath} \neq j$ and over some discrete time interval of size $T_r+1 \in \N$. Subsequently, these virtual sensors are used to compute \textit{pressure residuals} of measured and predicted pressure to train an ensemble classifier that is able to predict whether a leakage is present in the \gls{WDN} at the time of the used residual (cf. \cite{Isermann2006FaultDiagnosis}).

\subsubsection{Virtual Sensors}
The virtual sensors $f_j^r: \R^{d_r} \rightarrow \R$ for each node $j \in \{1,...,d\}$ and $d_r := d-1$ are linear regression models trained on leakage free training data  $\D_j^r = \{ (\overline{p}_{\neq j}(t_i), p_j(t_i)) \in \R^{d_r} \times \R ~|~ i = 0,...,n_r \}$. More precisely, $y(t_i) = 0 \in \Y$ holds for all realisations $i = 0,...,n_r$ of $Y$ and the inputs are given by the rolling means $\overline{p}_{\neq j}(t_i) := (T_r+1)^{-1} \sum_{\iota=0}^{T_r} p_{\neq j}(t_i - \iota\delta)$ at all nodes except the node $j$, which is the only preprocessing required for the training pipeline (cf. \cite{Artelt2022Explainability_ThresholdBasedLeakageDetection}).

\subsubsection{Ensemble Leakage Detection}
Standard leakage detection methods rely on the residuals $r_j(t_i) := |p_j(t_i) - f_j^r(\overline{p}_{\neq j}(t_i))| \in \R_+$ we obtain from the true pressure measurements $p(t_i) \in \R^d$ and the virtual sensor predictions $f_j^r(\overline{p}_{\neq j}(t_i)) \in \R$ for each sensor node $j \in \{1,...d\}$ and (possibly unseen) times $t_i \in \R$ (cf. \cite{Isermann2006FaultDiagnosis}).

A simple detection method performing good on standard benchmarks is the \textit{threshold-based ensemble classification} introduced by \cite{Artelt2022Explainability_ThresholdBasedLeakageDetection}: Without any further training, we can define a classifier $f_j^c: \R_+ \rightarrow \Y$ by
\begin{align*}
    f_j^c(r_j(t_i))
    =
    f_j^c(r_j(t_i), \theta_j) 
    := 
    \1_{ \left \{
    r_j(t_i) ~>~ \theta_j
    \right \} }
\end{align*}

\noindent for each sensor node $j \in \{1,...d\}$ and a node-dependant hyperparameter $\theta_j \in \R_+$.
We easily obtain an ensemble classifier $f^c: \X \rightarrow \Y$, called the H-method, with \textbf{H}yperparameter $\Theta := (\theta_j)_{j=1,...,d} \in \X$ for $\X := \R_+^{d_c}$ and $d_c := d$ that predicts whether there is a leakage present in the \gls{WDN} at time $t_i \in \R$ or not, defined by
\begin{align}
\label{align_EnsembleClassifier_LeakageDetection_ThresholdBased}
    f^c(r(t_i))
    =
    f^c(r(t_i), \Theta) 
    :=
    \1_{ \left \{
    \sum_{j=1}^{d_c} f_j^c(r_j(t_i)) ~\geq~ 1
    \right \} }.
\end{align}

\subsubsection{Evaluation} 
\label{subsubsection_Evaluation_In_LeakageDetectionInWaterDistributionNetworks}

We evaluate the H-method in terms of general performance, measured by \gls{ACC}, and in terms of fairness, measured by disparate impact as well as equal opportunity score \gls{DI} and \gls{EO}, respectively (cf. eq. \eqref{align_FairnessMeasure_DisparateImpact_Multi} and \eqref{align_FairnessMeasure_EqualOpportunity_Multi}) in section \ref{subsection_ExperimentalResults_In_LeakageDetectionInWaterDistributionNetworks}, after introducing the application domain and data set.

\subsection{Application Domain and Data Set}
\label{subsection_ApplicationDomainAndDataSet}

One key contribution of this work is to introduce the notion of fairness in the application domain of \glspl{WDN}. The \gls{WDN} considered is \textit{Hanoi} (cf. \cite{Santos2022PressureSesorPlacement_LeakDetection}) displayed in figure \ref{figure_HanoiWithSensorNodesAndProtectedGroups}. It consists of 32 nodes and 34 links. 

To evaluate the H-method presented in section \ref{subsection_Methodology_In_LeakageDetectionInWaterDistributionNetworks} on Hanoi, we generate pressure measurements with a time window of $\delta=10$min. using the atmn toolbox (cf. \cite{Vaquet2023atmn}).
The pressure is simulated at the sensor nodes displayed in figure \ref{figure_HanoiWithSensorNodesAndProtectedGroups} and for different leakage scenarios, which differ in the leakage location and size. As the \gls{WDN} is relatively small, we are able to simulate a leakage at each node in the network and for three different diameters $d \in \{5,10,15\}$cm.
For the preprocessing according to section \ref{subsection_Methodology_In_LeakageDetectionInWaterDistributionNetworks}, we choose $T_r = 2$ such as \cite{Artelt2022Explainability_ThresholdBasedLeakageDetection} do.

\begin{figure}[h!]
\centering
\includegraphics[width=0.75\textwidth]{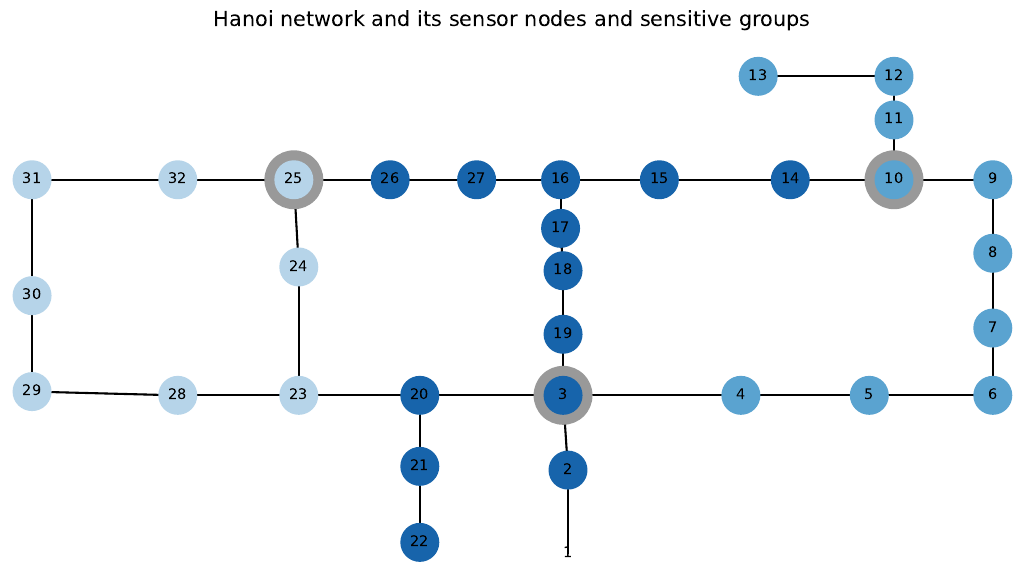}
\caption{The Hanoi \gls{WDN}, its sensor nodes (IDs 3, 10 and 25) and the protected groups, each highlighted in another color (group 1 on the left side in light shade, group 2 in the middle in dark shade, group 3 on the right side in middle shade). The sensor nodes are colored in the same color of the protected group they belong to and highlighted with a grey circle.}
\label{figure_HanoiWithSensorNodesAndProtectedGroups}
\end{figure} 

\noindent The question arises how the leakage detection is related to fairness. Knowing that each node of the network corresponds to a group of consumers, a natural question is whether these local groups benefit from the \gls{WDN} and its related services in equal degree. To ensure that the algorithms that will be presented in section \ref{subsection_Methodology_In_FairnessEnhancingLeakageDetectionInWaterDistributionNetworks} scale to larger \glspl{WDN}, we do not consider single nodes but groups of nodes in the \gls{WDN} as protected groups in terms of fairness. Then, for all $k = 1,...,K$, we define the sensitive feature $S_k \in \{0,1\}$ to give answer to the question whether (1) or whether not (0) a leakage is active in such a protected group $k$.  In terms of equal service one would expect an equally good detection of leakages independent on the leakage location, i.e., the protected group. For Hanoi, we work with $K=3$ different groups, also displayed in figure \ref{figure_HanoiWithSensorNodesAndProtectedGroups}.

Given \textit{this} definition of sensitive features $S_k$ for $k=1,...,K$ in the \gls{WDN}, we obtain the following important results with regard to the notions of fairness.

\begin{lemma}[Equivalence of disparate impact and equal opportunity in \glspl{WDN}]
\label{lemma_EOEquivalenceDI}
    Let $S_k$ be the sensitive feature describing whether a leakage is active in the protected group $k$ of the \gls{WDN} for each $k = 1,...,K$. Moreover, let $\epsilon, \tilde{\epsilon} \in [0,1]$ and define $\max_k := \max_{ k \in \{1,...,K\} } \Prob(\hat{Y}=1 ~|~ S_{k}=1)$.
    
    \begin{enumerate}
        \item If disparate impact holds with $\epsilon$, equal opportunity holds with $\tilde{\epsilon} = \epsilon \max_k$. 

        \item If equal opportunity holds with $\tilde{\epsilon}$, disparate impact holds with $\epsilon = \tilde{\epsilon} ( \max_k )^{-1}$. 
    \end{enumerate}
\end{lemma}

\begin{proof}
    First of all, note  that for any $k \in \{1,...,K\}$ and $\omega \in \Omega$, $S_k(\omega)=1$ implies $Y(\omega)=1$ by definition of the sensitive feature $S_k$. Therefore, $\{Y=0, S_k=1\}$ is empty. Subsequently, we obtain 
    $\{Y=1, S_k=1\} 
    = \{S_k = 1\}$  
    and thus, 
    $\Prob(\hat{Y}=1 ~|~ Y=1, S_{k}=1) = \Prob(\hat{Y}=1 ~|~ S_k=1)$.
    
    Secondly, we also define $\min_k := \min_{ k \in \{1,...,K\} } \Prob(\hat{Y}=1 ~|~ S_{k}=1)$. We then obtain $\DI = \frac{\min_k}{\max_k}$ and, together with the first observation, $\EO = \max_k - \min_k$.

    Now the rest follows by simple equivalent transformations.
\end{proof}

\begin{corollary}
\label{corollary_EOEquivalenceDI}
    Given the setting of lemma \ref{lemma_EOEquivalenceDI},

    \begin{enumerate}
        \item $\EO = \tilde{\EO}$ for $\tilde{\EO} := (1 - \DI ) \cdot \max_k$ and

        \item $\DI = \tilde{\DI}$ for $\tilde{\DI} := 1 - \frac{\EO}{\max_k}$ holds.
    \end{enumerate}
\end{corollary}

\begin{proof}
    This is a direct consequence of lemma \ref{lemma_EOEquivalenceDI}, where we choose $\epsilon := 1 - \DI$ in setting 1 and $\tilde{\epsilon} := \EO$ in setting 2, and where we can work with equalities instead of estimations.
\end{proof}

\subsection{Experimental Results and Analysis: Residual-Based Ensemble Leakage Detection Does Not Obey Fairness}
\label{subsection_ExperimentalResults_In_LeakageDetectionInWaterDistributionNetworks}

In table \ref{table_ResultsH-Method}, the results of the H-method presented in section \ref{subsection_Methodology_In_LeakageDetectionInWaterDistributionNetworks} are shown. The hyperparameter $\Theta \in \X = \R_+^{d_c}$ is chosen manually per diameter $d$ such that the test accuracy is close to maximal. On the one hand, we see that the method in general performs better the larger the leakage size is (measured in \gls{ACC}), as larger leakages are associated with larger pressure drops. Note that the method is capable of detecting even small leakages with high accuracy in larger (and therefore, more realistic) \glspl{WDN} (cf. \cite{Artelt2022Explainability_ThresholdBasedLeakageDetection}).

\begin{table}
    \centering
    \caption{Results of the H-method with 
    $\max_k  = \max_{ k } \Prob(\hat{Y}=1 ~|~ S_{k}=1)$ 
    and 
    $\min_k =  \min_{ k } \Prob(\hat{Y}=1 ~|~ S_{k}=1)$ 
    according to (the proof of) lemma \ref{lemma_EOEquivalenceDI}. Moreover, disparate impact score 
    $\DI$
    and
    $\tilde{\DI}$ 
    as well as equal opportunity score
    $\EO$
    and
    $\tilde{\EO}$ 
    according to equation \eqref{align_FairnessMeasure_DisparateImpact_Multi}, \eqref{align_FairnessMeasure_EqualOpportunity_Multi} and corollary \ref{corollary_EOEquivalenceDI}.2, \ref{corollary_EOEquivalenceDI}.1, respectively.}
    \label{table_ResultsH-Method}
    \begin{tabular}{l||c|c|c|c|c|c|c}
    $d$ & $\ACC$ & $\max_k$ & $\min_k$ & $\DI$ & $\EO$ & $\tilde{\DI}$ & $\tilde{\EO}$\\
    \hline\hline
    5 & 0.6223 & 0.8468 & 0.4880 & 0.5763 & 0.3558 & 0.5763 & 0.3588\\
    \hline
    10 & 0.7998 & 0.9983 & 0.6372 & 0.6383 & 0.3611 & 0.6383 & 0.3611\\
    \hline
    15 & 0.8837 & 1.0000 & 0.6402 & 0.6402 & 0.3598 & 0.6402 & 0.3598
    \end{tabular}
\end{table}

\noindent On the other hand, 
we see that the method is unfair in terms of disparate impact score \gls{DI}, where a value of 0.8 or larger is desirable (cf. \cite{Zafar2017FairClassification_InProcess_CovarianceConstraints}), and equal opportunity score \gls{EO}. 
However, the experimental evaluation confirms the mathematical findings of corollary \ref{corollary_EOEquivalenceDI} by
comparing the corresponding columns in table \ref{table_ResultsH-Method}.
This also justifies that in our setting, the usage of one of the two measures is sufficient. Therefore, from now on, we work with disparate impact score \gls{DI} only.

\section{Fairness-Enhancing Leakage Detection in Water Distribution Networks}
\label{section_FairnessEnhancingLeakageDetectionInWaterDistributionNetworks}

Motivated by the result that the standard leakage detection method presented in section \ref{subsection_Methodology_In_LeakageDetectionInWaterDistributionNetworks}  does not satisfy the notions of fairness, as another main contribution of this work, we modify this H-method to enhance fairness as introduced in section \ref{section_FairnessInMachineLearning}. 
Given virtual sensors $f_j^r$ for all $j = 1,...,d$ and resulting residuals $r(t_i) \in \X = \R_+^{d_c}$ (cf. section \ref{subsection_Methodology_In_LeakageDetectionInWaterDistributionNetworks}) as well as labels $y(t_i) \in \Y = \{0,1\}$ for times $t_i \in \R$, 
we can turn the choice of the hyperparameter $\Theta := (\theta_j)_{j=1,...,d} \in \X$ of the ensemble classifier $f^c$ (cf. eq. \eqref{align_EnsembleClassifier_LeakageDetection_ThresholdBased}) into an \gls{OP} 
with corresponding function space $\Hyp := \{ f^c:\X \rightarrow \Y, ~ r \mapsto f^c(r,\Theta) ~|~ \Theta \in \X \}$. 
In the following section, we therefore present different, in contrast to the H-method optimization-based, methods that aim at optimizing the parameter $\Theta \in \X$ in order to obtain an optimal ensemble classifier $f^c(\cdot, \Theta_{\text{opt.}}) \in \Hyp$. 
These methods on the one hand are further baselines, where treating the modelling problem as an \gls{OP} enables us to optimize the result of the H-method itself without fairness considerations. 
On the other hand, we consider fairness-enhancing methods, where the parameter $\Theta \in \X$ needs to be optimized such that the resulting ensemble classifier is as accurate and fair on the given training data as possible.

\subsection{Methodology}
\label{subsection_Methodology_In_FairnessEnhancingLeakageDetectionInWaterDistributionNetworks}

The following methods define training algorithms based on labeled training data 
$\D^c = \{ (r(t_i), y(t_i)) \in \X \times \Y ~|~ i = 1,...,n_c \}$\footnote{
    In practise, we train and test the (ensemble) classifier(s) on unseen data for times $i \geq n_r+1$. However, for the sake of readability, we choose the indices $i=1,...,n_c$ instead of $i=n_r+1,...,n_c$ here.
} 
for an $n_c > n_r$, which also holds data based on leaky states of the \gls{WDN}. For simplicity, we omit the dependence of all (loss) functions on the training data.

\subsubsection{Optimizing Loss with Fairness Constraints}
\label{subsubsection_OptimizingLossWithFairnessContraints}

In general, a learning problem can be phrased as an \gls{OP}, where the objective is to minimize some suitable loss function $L: \X \rightarrow \R$ with respect to the parameter $\Theta \in \X$, i.e., 
\begin{align}
\label{align_OptimizationProblem_unconstrained}
\begin{split}
\begin{cases}
    \min_{\Theta \in \X} 
    ~ L(\Theta).
\end{cases}
\end{split}
\end{align}

\noindent The advantage of redefining the choice of hyperparameters $\Theta$ (H-method) as an \gls{OP} is that we can now extend this \gls{OP} by fairness constraints, which can be given by side constraints $C_k: \X \rightarrow \R$ of the underlying \gls{OP}:
\begin{align}
\label{align_OptimizationProblem_constrained}
\begin{split}
\begin{cases}
    \min_{\Theta \in \X} 
    &~ L(\Theta),
    \\
    \text{s.t.} 
    &~ C_{k}(\Theta) \geq 0 ~ \forall k = 1,...,\hat{K}.
\end{cases}
\end{split}
\end{align}

\noindent
\textit{Choice of Loss Functions}
In view of the notions of fairness (cf. section \ref{section_FairnessInMachineLearning}), an intuitive and by means of linearity easily to differentiate loss function is given by the difference of the \gls{FPR} and the \gls{TPR}, i.e., 
$L_1(\Theta) := - \TPR(\Theta) + \FPR(\Theta)$. 
Another classical evaluation score which we can use as a loss function is the accuracy $L_2(\Theta) := - \ACC(\Theta)$.

\textit{Choice of Fairness Constraints}
In terms of fairness constraints, \cite{Zafar2017FairClassification_InProcess_CovarianceConstraints} introduce the covariance between a single binary sensitive feature and the signed distance of a feature vector and the decision boundary of a convex margin-based classifier as a proxy for fairness measurements. We adapt this idea to our setting by considering the covariance of each sensitive feature and replacing the signed distance by the prediction of the ensemble classifier $\hat{Y} = f^c(X, \Theta)$. Using that $\hat{y}(t_i) = f^c(r(t_i), \Theta)$ holds for all realisations $i=1,...,n_c$, for all sensitive features $S_k$ for $k=1,...,K$, the \textit{empirical} covariance is given by
\begin{align}
\label{align_empiricalCovariance_sensitiveFeatureAndPrediction}
\begin{split}
    \Cov_{\text{emp.}}(S_k,\hat{Y})
    =
    \frac{1}{n_c} \sum_{i=1}^{n_c} 
    (s_k(t_i) - \overline{s_k}) \cdot f^c(r(t_i), \Theta).
\end{split}
\end{align}

\begin{remark}
\label{remark_whyCovariance}
    The usage of the (empirical) covariance as a proxy for fairness is based on the idea that fairness of a machine learning model $\hat{Y}$ can be interpreted as the assumption of $\hat{Y}$ being independent of the sensitive feature $S$ (cf. 
    \cite{Mehrabi2021Overview_SurveyOnFairness}), or in our case, each of the sensitive features $S_k$ for $k = 1,...,K$. As independence of two random variables implies their covariance being equal to zero, the latter can be interpreted as a necessary condition for fairness. 
    Therefore, we consider the covariance of each sensitive feature $S_k$ for $k = 1,...,K$ and the prediction of the ensemble classifier $\hat{Y} = f^c(X, \Theta)$, which by linearity is given by
    \begin{align*}
        \Cov(S_k,\hat{Y})
        &=
        \Exp((S_k - \Exp(S_k)) \cdot (\hat{Y} - \Exp(\hat{Y}))
        =
        \Exp((S_k - \Exp(S_k)) \cdot \hat{Y}).
    \end{align*}

    \noindent However, as the probability measure $\Prob (S_k, \hat{Y})^{-1}$ on $\Y \times \Y$ is unknown, we replace it by its empirical approximation $\frac{1}{n_c} \sum_{i=1}^{n_c} \delta_{(s_k(t_i), \hat{y}(t_i))}$ and obtain the empirical covariance \eqref{align_empiricalCovariance_sensitiveFeatureAndPrediction}.
    As in practise, an exact value of zero will rarely be achieved, enforcing the empirical covariance to be small is a reasonable fairness proxy.
\end{remark}

\noindent Assuming that a comparatively high (empirical) covariance (in either positive or negative direction) between a sensitive feature $S_k$ for $k \in \{1,...,K\}$ and the model's prediction $\hat{Y} = f^c(X, \Theta)$ implies a significant difference in the relative amount of positive predictions in contrast to the remaining sensitive features leads to the idea of constraining the absolute value of the (empirical) covariance as a side constraint in the above considered \gls{OP} (cf. eq. \eqref{align_OptimizationProblem_unconstrained}) (cf. \cite{Zafar2017FairClassification_InProcess_CovarianceConstraints}). 

Motivated by that, we require 
$
    \Cov_{\text{emp.}}(S_k,\hat{Y}) \leq c 
    \text{ and } 
    \Cov_{\text{emp.}}(S_k,\hat{Y}) \geq -c
$
or, equivalently formulated in standard form, 
\begin{align*}
    C_k(\Theta) := c - \Cov_{\text{emp.}}(S_k,\hat{Y}) \geq 0 
    \text{ and }
    C_k(\Theta) := c + \Cov_{\text{emp.}}(S_k,\hat{Y}) \geq 0
\end{align*}

\begin{wraptable}{r}{0.4\textwidth}
\vspace{-10pt} 
    \centering
    \caption{Overview of the proposed methods.}
    \label{table_methods}
    \begin{tabular}{l||c|c}
    Method & Loss & Constraints\\
    \hline\hline 
    T-F-PR & $L_1$ & - \\
    \hline
    T-F-PR+F & $L_1$ & emp. Cov.\\
    \hline
    ACC & $L_2$ & - \\
    \hline
    ACC+F & $L_2$ & emp. Cov.\\
    \end{tabular}
\vspace{-10pt} 
\end{wraptable}

\noindent to hold for all $k = 1,...,K$ (i.e., $\hat{K} = 2K$ in equation \eqref{align_OptimizationProblem_constrained}). Hereby, the hyperparameter $c \in [0,\infty)$ regulates how much the covariance's absolute value is bounded and therefore, the desired fairness (cf. remark \ref{remark_whyCovariance}).

\textit{Explicit Methods}
The resulting methods as a combination of used loss function with or without the fairness-enhancing side constraint (cf. \gls{OP} \eqref{align_OptimizationProblem_unconstrained} or \eqref{align_OptimizationProblem_constrained}) deliver two baseline and two fairness-enhancing ensemble leakage detection algorithms, summarized in table \ref{table_methods}.

\textit{Differentiable Approximation of the Learning Problems}
Loss function and side constraint (cf. eq. \eqref{align_empiricalCovariance_sensitiveFeatureAndPrediction}) clearly depend on the model's prediction $\hat{Y} = f^c(X, \Theta)$ resp. $y(t_i) = f^c(r(t_i), \Theta)$ for all $i=1,...,n_c$. However, in view of the model's definition (cf. eq. \eqref{align_EnsembleClassifier_LeakageDetection_ThresholdBased}), $f^c$ is not differentiable with respect to $\Theta$. 
To make $\hat{Y} = f(X, \cdot)$ differentiable, we approximate each indicator function $\1_{\{v > 0\}}$ by the sigmoid function $\sgd_b(v) = (1 + \exp^{-bv})^{-1}$ 
with hyperparameter $b \in \R_+$. 
All in all, we obtain a differentiable \gls{OP} by replacing the ensemble classifier's prediction $f^c(r(t_i), \Theta)$ (cf. eq. \eqref{align_EnsembleClassifier_LeakageDetection_ThresholdBased}) by 
\begin{align}
\label{align_EnsembleClassifier_LeakageDetection_ThresholdBased_Approximiated}
    \hat{f}^c(r(t_i), \Theta) :=
    \sgd_b 
    \big(
    \textstyle \sum_{j=1}^{d}
    \sgd_b(r_j(t_i) - \theta_j)
    - T
    \big)
\end{align}

\noindent for all $i = 1,...,n_c$, where we replace the threshold 1 of the exact ensemble classifier $f^c$ by a hyperparameter $T \in [0,1]$ to handle the insecurity of the sigmoid function around zero. 
Then, by expressing the losses $L_1 = -\TPR + \FPR$ and $L_2 = -\ACC$ by
\begin{align*}
    L_1(\Theta) 
    &=
    - \frac{
    \sum_{i = 1}^{n_c} y(t_i) \cdot f^c(r(t_i), \Theta)
    }{
    \sum_{i = 1}^{n_c} y(t_i)
    }
    +
    \frac{
    \sum_{i = 1}^{n_c} (1 - y(t_i)) \cdot f^c(r(t_i), \Theta)
    }{
    \sum_{i = 1}^{n_c} (1 - y(t_i))
    },
    \\
    L_2(\Theta) 
    &=
    \frac{
    \sum_{i = 1}^{n_c} y(t_i) \cdot f^c(r(t_i), \Theta)
    +
    \sum_{i = 1}^{n_c} (1 - y(t_i)) \cdot (1 - f^c(r(t_i), \Theta))
    }{
    n_c
    },
\end{align*}

\noindent their approximated versions using $\hat{f}^c$ instead of $f^c$ will be differentiable with respect to $\Theta$ as well, 
and so is the empirical covariance (cf. eq. \eqref{align_empiricalCovariance_sensitiveFeatureAndPrediction}) when using $\hat{f}^c$ instead of $f^c$. The resulting approximated \glspl{OP} can therefore be optimized with a gradient-based optimization technique.

\subsubsection{Optimizing Fairness with Accuracy Constraints}
\label{subsubsection_OptimizingFairnessWithAccuracyContraints}

Instead of optimizing some loss function $L$ under some fairness side constraints, \cite{Zafar2017FairClassification_InProcess_CovarianceConstraints} suggest to optimize a fairness proxy under loss constraints. They use the covariance as a proxy while constraining the training loss by some percentage of the optimal loss obtained when training without fairness considerations. As a variation, we use the disparate impact score $\DI$ directly as a loss function and the accuracy $\ACC$ for the constraint. The resulting DI+ACC-method is therefore given by
\begin{align}
\label{align_OP_FairnesslossWithLossConstraints}
\begin{split}
\begin{cases}
    \min_{\Theta \in \X} 
    &~ -\DI(\Theta),
    \\
    \text{s.t.} 
    &~ \ACC(\Theta) \geq (1 - \lambda) \ACC_{\text{opt.}}.
\end{cases}
\end{split}
\end{align}

\noindent The hyperparameter $\lambda \in [0,1]$ hereby regulates how much the obtained accuracy $\ACC(\Theta)$ is allowed to differ from the optimal accuracy $\ACC_{\text{opt.}}$ received in the ACC-method (cf. table \ref{table_methods}).

In contrast to the methods proposed in section \ref{subsubsection_OptimizingLossWithFairnessContraints}, we like to test the \gls{OP} \eqref{align_OP_FairnesslossWithLossConstraints} as a non-differentiable \gls{OP}, which therefore requires a non-gradient-based optimization technique.

\subsubsection{Evaluation}
\label{subsubsection_Evaluation_In_FairnessEnhancingLeakageDetectionInWaterDistributionNetworks}

We evaluate all presented methods, i.e., the standard H-method 
(section \ref{subsection_Methodology_In_LeakageDetectionInWaterDistributionNetworks}), 
the optimization-based baselines T-F-PR- and ACC-method as well as the fairness-enhancing T-F-PR+F-, ACC+F- 
(cf. section \ref{subsubsection_OptimizingLossWithFairnessContraints} and table \ref{table_methods}) 
and DI+ACC-method 
(cf. section \ref{subsubsection_OptimizingFairnessWithAccuracyContraints}), 
as we did in section \ref{subsubsection_Evaluation_In_LeakageDetectionInWaterDistributionNetworks}.

\subsection{Experimental Results and Analysis}
Based on the pressure measurements in the Hanoi \gls{WDN} as introduced in section \ref{subsection_ApplicationDomainAndDataSet} and the resulting residuals, we test all six methods introduced in section \ref{subsection_Methodology_In_LeakageDetectionInWaterDistributionNetworks} (H-method) and section \ref{subsection_Methodology_In_FairnessEnhancingLeakageDetectionInWaterDistributionNetworks}
(T-F-PR-, ACC-, T-F-PR+F, ACC+F, DI+ACC-method, also see Evaluation in section \ref{subsubsection_Evaluation_In_FairnessEnhancingLeakageDetectionInWaterDistributionNetworks}) per diameter $d$ in practise.

\subsubsection{Setup}
\textit{H-method} 
We use the H-method presented in section \ref{subsection_Methodology_In_LeakageDetectionInWaterDistributionNetworks} and tested in section \ref{subsection_ExperimentalResults_In_LeakageDetectionInWaterDistributionNetworks} as a baseline. Subsequently, we use the hyperparameter found here as an initial parameter $\Theta_0 \in \X$ for the remaining optimization-based methods. 

\textit{Gradient-Based Methods} 
While the T-F-PR- and the ACC-method are used as another baseline, the remaining methods are fairness-enhancing methods. 
The magnitude of fairness can be regulated by a hyperparameter: 
The T-F-PR+F- and ACC+F-method ensure fairness by bounding the empirical covariance of each sensitive feature and the models approximated prediction (cf. eq. \eqref{align_empiricalCovariance_sensitiveFeatureAndPrediction} and \eqref{align_EnsembleClassifier_LeakageDetection_ThresholdBased_Approximiated}) by the hyperparameter $c \in [0, \infty)$. 
In addition, for all these methods, we choose $b=100$ and $T=0.8$.

\textit{Non-Gradient-Based Method} 
In contrast, the DI+ACC-method regulates fairness by different choices of the hyperparameter $\lambda \in [0,1]$ that controls how much loss in accuracy is allowed while increasing fairness.

\textit{Transforming Constraint \glspl{OP} in Non-Constraint \glspl{OP}} 
For all \glspl{OP}, we use the log-barrier method (cf. \cite{Nocedal2006NumericalOptimization}) to transform the constrained \gls{OP} into a non-constrained one and tune the regularization hyperparameter per method. 

\textit{Optimization Techniques} 
For the differentiable \glspl{OP} (T-F-PR-, ACC-, T-F-PR+F- and ACC+F-method), we use BFGS (cf. \cite{Nocedal2006NumericalOptimization}) to find the optimal parameter $\Theta_{\text{opt.}} \in \X$. For the non-differentiable \gls{OP} (DI+ACC-method), we use Downhill-Simplex-Search, also known as Nelder-Mead (cf. \cite{Gao2012NelderMead}). Each method is trained per diameter $d$ on 40\% of the data and evaluated on the remaining data.

\begin{wrapfigure}{r}{0.45\textwidth}
\vspace{-25pt} 
\begin{minipage}{\linewidth}
    \centering
    \includegraphics[scale=0.44]{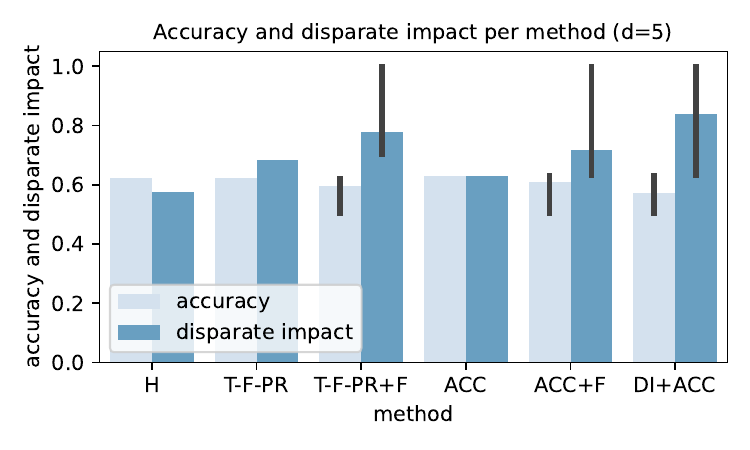}
    \hfill
    \includegraphics[scale=0.44]{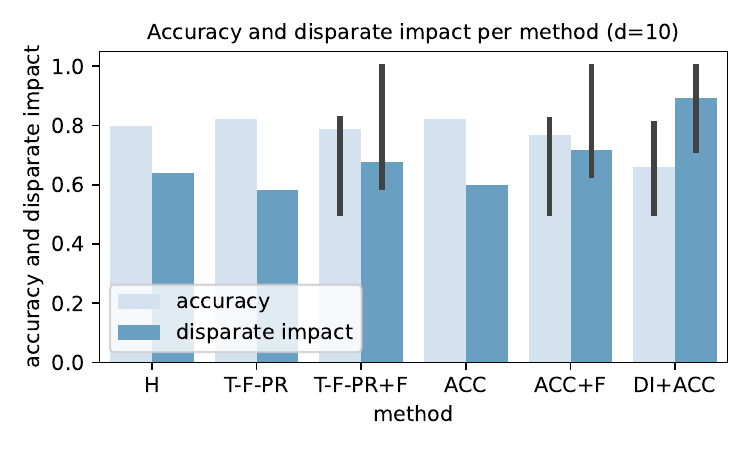}
    \hfill
    \includegraphics[scale=0.44]{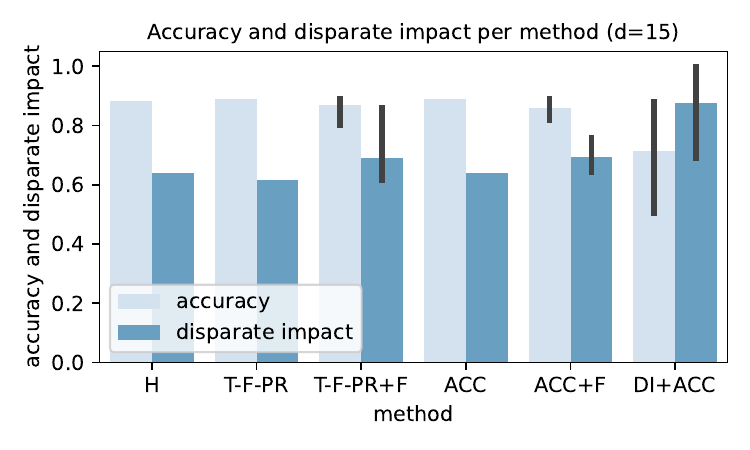}
\end{minipage}
\caption{Accuracy and disparate impact score per method and leakage diameter.}
\label{figure_ACCandDIperMethod}
\vspace{-10pt} 
\end{wrapfigure}

\subsubsection{Results and Analysis}
\textit{Increasing Fairness}
In figure \ref{figure_ACCandDIperMethod}, we see the performance of each ensemble classifier measured in accuracy and disparate impact score.
For the fairness-enhancing methods, we test different hyperparameters, causing error bars for these methods. 
We start with a hyperparameter $c$ and $\lambda$ that causes an accuracy of 0.5 and disparate impact score of 1.0 whenever possible and in- resp. decrease the hyperparameter by 0.01 until the disparate impact score of the fairness-enhancing model achieves the disparate impact score of the corresponding baseline (T-F-PR for T-F-PR+F and ACC for ACC+F and DI+ACC). 
The height of the bars with error bars correspond to the mean accuracy and disparate impact score, achieved by each method over all hyperparameters tested. The error bars themselves reach from the lowest to the largest score of the two scores considered. 

We see that for all fairness-enhancing methods and all leakage sizes, the fairness-enhancing methods on average increase fairness, measured by disparate impact score, while on average, mostly decreasing accuracy by only some small percentage compared to their corresponding baselines. 
For $d=5$, all fairness-enhancing methods allow a large range of fairness improvement at cost of a small range in accuracy, which is only due to the relatively poor accuracy of the leakage detector in this scenario in general. For the other diameters, the ranges of fairness and accuracy are similarly large. Thus, one can say that fairness and overall performance are mutually dependent to about the same extent.  

\textit{The Coherence of Fairness and Overall Performance}
A more detailed visualization of how fairness is related to the overall performance of the model can be found in figure \ref{figure_CoherenceACCandDI}. 
For each tested hyperparameter $c$ and $\lambda$, respectively, depending on what fairness-enhancing method was used, the obtained disparate impact score is plotted together with the observed accuracy. For better readability, we split these observations by the leakage sizes tested.
\textcolor{white}{Some words to debug.} 
\begin{wrapfigure}{l}{0.45\textwidth}
\vspace{-15pt} 
\centering
\includegraphics[width=0.44\textwidth]{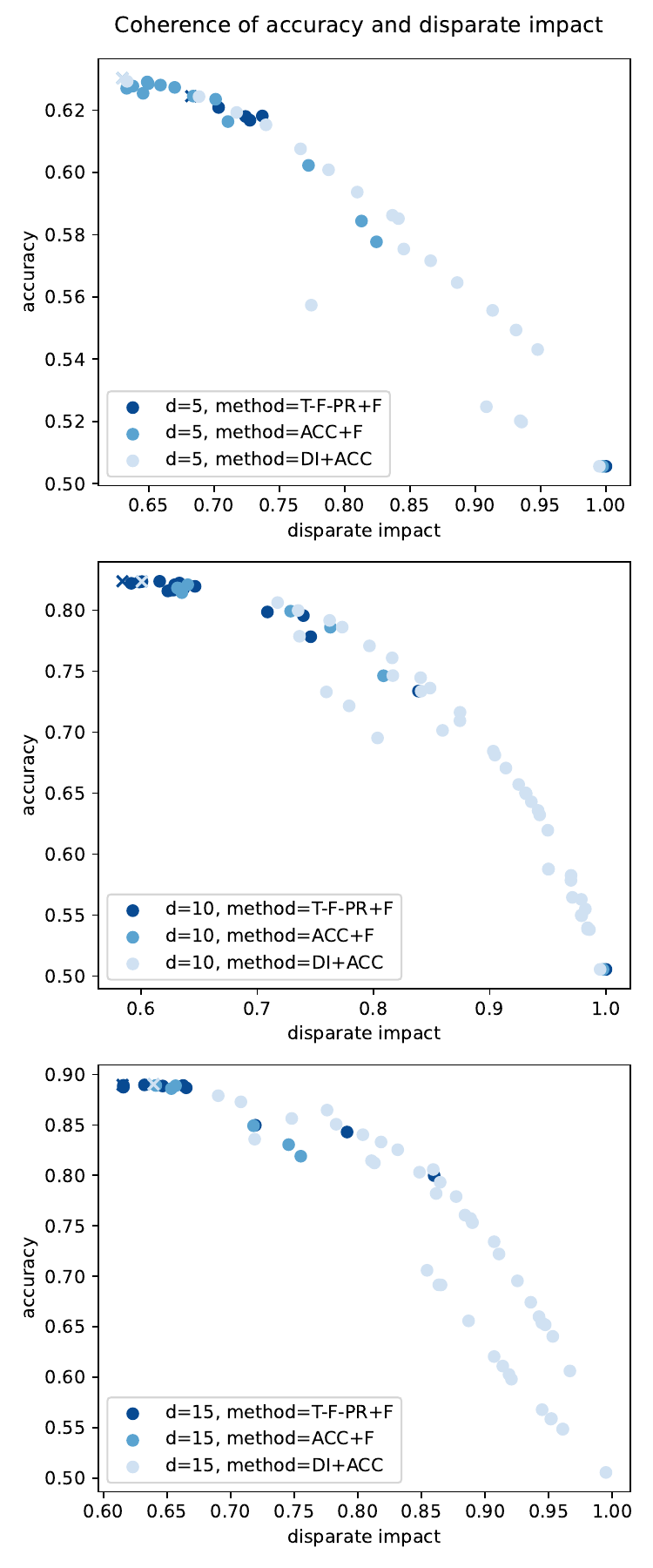}
\caption{Coherence of accuracy and disparate impact score for the different fairness-enhancing methods and different leakage sizes. The cross data points visualize the disparate impact score and accuracy of the non-fairness-enhancing baselines  (T-F-PR, dark blue, for T-F-PR+F and ACC, light blue, for ACC+F and DI+ACC).}
\label{figure_CoherenceACCandDI}
\vspace{-10pt} 
\end{wrapfigure}
\indent 
The characteristic curve that can be observed in all subimages is the so-called pareto front, visualizing that the increase in fairness is accompanied by the reduction in accuracy score and vice versa. 
A desired disparate impact score of about 0.8 can be achieved by a decrease of accuracy by approximately 0.03 - 0.05 points below the optimal accuracy obtained.
Hereby, both fairness and overall performance can be influenced by the fairness hyperparameters $c$ and $\lambda$, respectively. 
Deciding which choice of fairness hyperparameter is optimal is a difficult task that depends on the extent of the decisions of the underlying \gls{ML} model as well as legal requirements. Regarding legal requirements, by not using the sensitive features for the decision making of the algorithm, the methods presented can satisfy the legal definition of disparate treatment \textit{and} disparate impact (depending on the hyperparameter chosen) simultaneously.
\\ 
\indent 
Another observation is that the largest accuracies of the fairness-enhancing methods are approximately as good as the accuracy of their baseline methods while achieving equal or better fairness results. 
In opposite direction, perfect fairness of 1.0 can be achieved at a cost of the worst possible accuracy of 0.5.
While for the covariance-based algorithms (T-F-PR+F- and ACC+F-method), the jump in disparate impact and accuracy score is rather abrupt when reaching the extreme of (1.0, 0.5), the method relying on the optimization of fairness while constraining on the accuracy (DI+ACC-method) allows more fine-grained variations in both scores.
This is due to the fact that the hyperparameter $\lambda$ regulates the accuracy and not the fairness measure. The accuracy constraint is less sensitive to the log barrier method than the covariance constraints, since a too small choice of the hyperparameter $c$ quickly sets all punishment terms to infinity and thus outputs the trivial solution.
\newpage
\textit{The Influence of the Hyperparameters on Fairness and Overall Performance}
In figure \ref{figure_CoherenceACCandDIandHyperparameter}, we show how the hyperparameters are related to disparate impact and accuracy score. 
Each of the two scores is plotted against the used hyperparameter for all fairness-enhancing methods and leakage diameters tested. 

\begin{figure}[h!]
\centering
\includegraphics[width=0.95\textwidth]{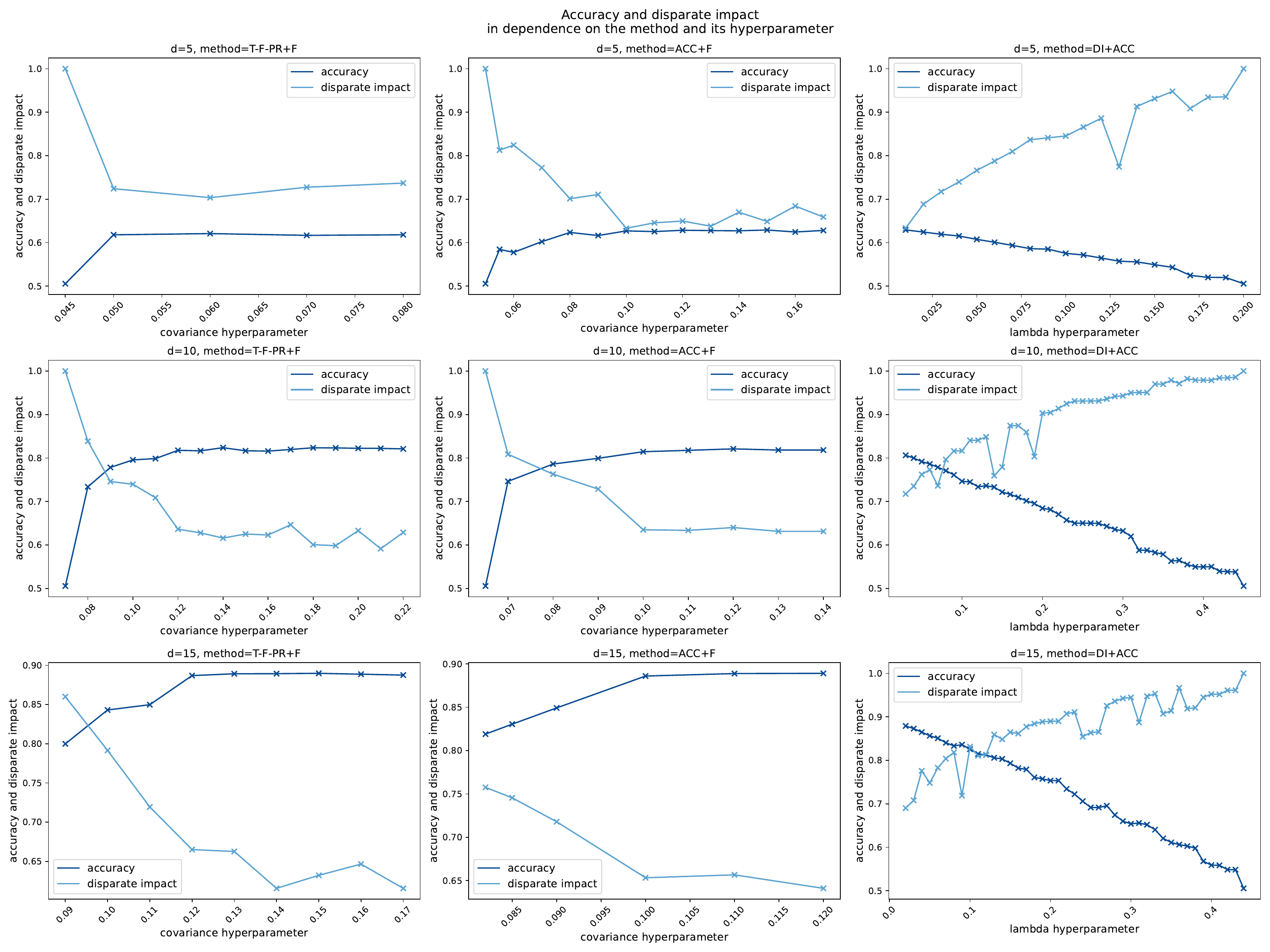}
\caption{Coherence of accuracy, disparate impact and the training hyperparameter.}
\label{figure_CoherenceACCandDIandHyperparameter}
\end{figure}

\noindent 
For the T-F-PR+F- and the ACC+F-method, the decrease of the hyperparameter $c$ is accompanied by the improvement of the fairness measure as well as the decrease of the performance measure. This can be explained by the intuition described before: 
A high empirical covariance of a sensitive feature and the prediction of the approximated ensemble model means that the relative number of positive predictions within the related group differs significantly from the relative number of positive predictions within a group with small covariance. Thus, the more the covariance is constrained, the less such extreme differences in the relative number of positive predictions across groups occur, leading to a better fairness score. In the case of disparate impact, therefore, a (better) higher score at the expense of a (worse) lower overall performance - compared to the overall performance that occurs in the unconstrained case or for a looser constraint, that is a larger bound $c$, - appears. 

In contrast, for the DI+ACC method the increase of the hyperparameter $\lambda$ is accompanied by the improvement of the fairness measure as well as the decrease of the performance measure 
due to the fact that a higher hyperparameter $\lambda$ allows a larger deviation of the optimal accuracy score. Thus, a worse accuracy is penalized less so that the fairness measure can be optimized to a larger extend. 

\textit{Non Optimality}
Last but not least, note that the non-optimal solutions and the local jumps recognized in figure \ref{figure_CoherenceACCandDI} and \ref{figure_CoherenceACCandDIandHyperparameter}, respectively, 
can be explained by the non-convexity of the objective functions. Therefore, the found solutions strongly depend on the initialized parameter $\Theta_0$ and might not correspond to the global optimum.

\section{Conclusion}
\label{section_Conclusion}

In this work, we introduced the notion of fairness in an application domain of high social and ethical relevance, namely in the field of \acrfullpl{WDN}. This required the extension of fairness definitions for a single binary sensitive feature to single non-binary or multiple, possibly even non-binary, sensitive features. We then investigated on the fairness issue in the area of leakage detection within \glspl{WDN}. We showed that standard approaches are not fair in the context of different groups related to the locality within the network. As a remedy, we presented methods that increase fairness of the ensemble classification model with respect to the introduced fairness notion while satisfying the legal notions of disparate treatment and disparate impact simultaneously. We empirically demonstrated that fairness and overall performance of the model are interdependent and the use of hyperparameters provides the ability to trade off fairness and overall performance. However, this trade off lies in the responsibility of the policy maker. 

To allow more fine-grained steps between improving fairness and decreasing overall performance in the presented covariance-based approaches, next steps would be to swap loss function and constraint to achieve similar results as in the approach with accuracy constraint. 
Moreover, the notion of fairness within the water domain is still in its beginning and extensions to more complex \glspl{WDN} as well as more powerful \gls{ML} algorithms is essential.

\section{Acknowledgments}
We gratefully acknowledge funding from the \gls{ERC} under the \gls{ERC} Synergy Grant Water-Futures (Grant agreement No. 951424).
\\\\
This work was first published in the proceedings of the 17th International Work-Conference on Artificial Neural Networks (IWANN) in \href{https://link.springer.com/chapter/10.1007/978-3-031-43085-5_10}{volume 14134 of Lecture Notes in Computer Science, pages 119--133, by Springer Nature in 2023} (cf. \cite{Strotherm2023FairClassification_InProcess_CovarianceConstraints}).



\newpage
\setglossarystyle{list}
\printglossary[title=Glossary, nonumberlist]
\addcontentsline{toc}{section}{Glossary}



\setglossarystyle{list}
\printglossary[type=\acronymtype, title=List of Abbreviations, nonumberlist]
\addcontentsline{toc}{section}{List of Abbreviations}



\newpage



    

\bibliography{main}





\end{document}